\documentclass[10pt,twocolumn,letterpaper]{article}

\usepackage[pagenumbers]{cvpr} 

\definecolor{cvprblue}{rgb}{0.21,0.49,0.74}
\usepackage[pagebackref,breaklinks,colorlinks,allcolors=cvprblue]{hyperref}

\def\confName{CVPR}
\def\confYear{2026}

\usepackage{dblfloatfix}

\usepackage{amsthm}

\newtheorem{thm}{Theorem}
\newtheorem{lem}[thm]{Lemma}

\title{End-to-end Feature Alignment: A Simple CNN with Intrinsic Class Attribution}

\author{Parniyan Farvardin\\
University of Miami\\
{\tt\small pxf291@miami.edu}
\and
David Chapman\\
University of Miami\\
{\tt\small dchapman@cs.miami.edu}
}

\begin{document}
\maketitle
\begin{abstract}
We present Feature-Align CNN (FA-CNN), a prototype CNN architecture with intrinsic class attribution through end-to-end feature alignment.     Our intuition is that the use of unordered operations such as Linear and Conv2D layers cause unnecessary shuffling and mixing of semantic concepts, thereby making raw feature maps difficult to understand.  We introduce two new order preserving layers, the dampened skip connection, and the global average pooling classifier head. These layers force the model to maintain an end-to-end feature alignment from the raw input pixels all the way to final class logits. This end-to-end alignment enhances the interpretability of the model by allowing the raw feature maps to intrinsically exhibit class attribution. 
We prove theoretically that FA-CNN penultimate feature maps are identical to Grad-CAM saliency maps.
Moreover, we prove that these feature maps slowly morph layer-by-layer over network depth, showing the evolution of features through network depth toward penultimate class activations. 
FA-CNN performs well on benchmark image classification datasets.  Moreover,  we compare the averaged FA-CNN raw feature maps against Grad-CAM and permutation methods in a percent pixels removed interpretability task.  We conclude this work with a discussion and future, including limitations and extensions toward hybrid models.
\end{abstract}

\section{Introduction}
\label{sec:intro}

\begin{figure*}[htbp]
  \centering
\includegraphics[width=1.0\textwidth]{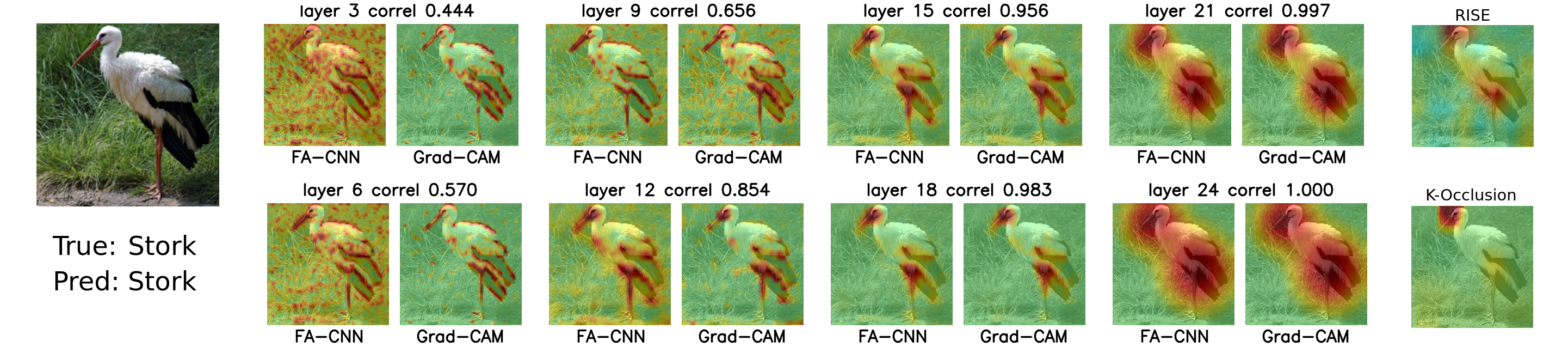}
    \caption{Comarison of averaged FA-CNN raw featuremaps (Left) versus Grad-CAM maps (Right).  Perturbation methods RISE and K-OCC are also shown (Far right).  Deepest layers are equivalent (correl 1.0), and middle layers show strong similarity.  Aligned features also show gradual morphing layer-by-layer through depth.  }
    \label{fig:billboard}
    \hfill
\end{figure*}

There is a strong and growing desire for neural network architectures that exhibit \textit{intrinsic interpretability}.  The goal of an \textit{intrinsically interpretable} model is to be able to learn a representation that a human can understand without the need for running another method to explain it.  In a perfect world, one ought to be able to plot the raw feature maps directly, and understand the meaning of the features.

Indeed, traditional CNN feature maps are not  \textit{Intrinsically Interpretable}.  Work in the areas of \textit{Network Dissection} and later \textit{Mechanistic Interpretability} have shown that CNN features comprise an \textit{entangled representation} \cite{bau2017network, chen2020concept, elhage2022toy}.  That in fact, features represent poly-semantic concepts.  The entanglement of seemingly unrelated semantic concepts is now considered to be the dominant root cause of the lack of interpretability of the feature representation \cite{bau2017network, chen2020concept, elhage2022toy}. Moreover, these entangled concepts often cross label boundaries, making the representation even more difficult to understand by a human \cite{liang2020training}.

So how can one improve the \textit{intrinsic interpretability} of CNN features?  There are two primary strategies toward this problem: \textit{Disentanglement}: an unsupervised approach to extract meaningful concepts from polysemantic features, and \textit{Class-Attribution}: a supervised approach to relate features to human-defined labels or classes.  In this work, we propose a novel strategy of \textit{end-to-end Feature alignment} through simple architectural modification.  We show that this approach effectively yields \textit{Class Attribution} throughout the depth of the CNN model.  In particular, we prove that Feature-Align CNN (FA-CNN) raw penultimate feature maps are identical to Grad-CAM saliency maps.  Furthermore, FA-CNN aligns features throughout network depth, maintaining a coherent ordering at every layer due to simple architectural constraints.    We prove that a rescaling of the FA-CNN features exhibits the \textit{Bounded Increment Property} through network Depth.  This property implies that one can observe the feature maps \textit{slowly morph} through network depth layer-by-layer from the projected inputs all the way to the penultimate attribution maps used for prediction.  Also notable is \textit{how} the features \textit{slowly morph}.  Qualitatively, the early layers appear to show a snapshot of what the model thinks about the target classes at that level of depth.  These features then morph toward the Grad-CAM saliency as additional semantic context becomes available.

These interpretability properties are clearly demonstrated in Figure \ref{fig:billboard} which shows an example image of a Stork from ImageNet.  Averaged FA-CNN raw features are compared with Grad-CAM saliency throughout network depth.  One can see that the penultimate features are identical (correl $1.0$), and that raw features morph slowly layer by layer due to bounded increment.  The middle layers also show high correlation and visually identify relevant parts of the stork.  Qualitatively the early layers show some excitement over the grassy background regions regions, which shows that the Stork concept requires moderate network depth to be able to distinguish from background.  After Layer 9, the model has enough semantic context to suppress these background features and by layer 12 onward the feature maps highly agree with Grad-CAM.  Although this gradual morphing shows something different from Grad-CAM in the early layers, it provides insight as to \textit{how much depth was necessary} for the model to understand the labeled concept and make its prediction.



\textit{Class Attribution} is highly desirable, and is a core area of interpretability research for vision models \cite{zhang2021survey, ibrahim2023explainable}.  Saliency maps produce a heatmap showing the regions of the image that are \textit{most important} for the vision model to predict of a given label.  Saliency maps have proven to be indispensable for practitioners to understand model predictions, and address model shortcomings for a wide range of datasets and applications.  Permutation methods are considered state-of-the-art in terms of the quality of the saliency map produced.  These methods are also model agnostic and gradient free, but they run extremely slowly requiring thousands of forward passes of the model to produce even a single saliency map \cite{simonyan2013deep, petsiuk2018rise, ribeiro2016should}.  Popular permutation methods include K-Occlusion \cite{simonyan2013deep}, RISE \cite{petsiuk2018rise}, LIME \cite{ribeiro2016should}, and SHAP \cite{lundberg2017unified}.  We compare the performance of our averaged penultimate feature maps against K-Occlusion and RISE as a means of evaluating class attribution quality.

Another category of \textit{Class Attribution} methods are the gradient based methods, including Grad-CAM and it's descendent Grad-CAM++ \cite{selvaraju2017grad, chattopadhay2018grad}.  These methods are much faster than permutation methods but they still typically require a separate forward pass of the model per label category in order to calculate gradients, which is still relatively slow.  Our proposed FA-CNN is an exception, as we prove that a simple averaging of the penultimate feature maps yields class attribution saliency maps for all classes simultaneously.  Moreover, we prove (and demonstrate) that FA-CNN features gradual morph over the network depth which provides insight into how these attribution maps are constructed by the network over each successive layer.  

A separate branch of research toward \textit{Class Attribution} involves the design of intrinsically interpretable architectures. 
Within this space there are several subcategories of approaches.  A notable branch of techniques are the Concept Bottleneck Networks (CBNs) \cite{koh2020concept, sawada2022concept, pittino2023hierarchical, kim2023probabilistic, shang2024incremental, wang2025mvp}.  CBNs exhibit a Concept Bottleneck Layer (CBL) which forces the model to predict a set of interpretable \textit{concepts}, and use only these \textit{concepts} to make the final prediction.  Recent work in CBNs is beginning to address the problem of unsupervised \textit{concept discovery} in addition to supervised concept bank \cite{sawada2022concept, shang2024incremental}.  

A notable limitation of CBM, is that the CBL is typically just a vector, and not a full saliency map. \citet{pittino2023hierarchical} attempts to provide concept localization but this requires a multi-level system design.  Another notable question is \textit{how} did the concept vector come to be?  The CBL is typically the result of a full black-box network.  \citet{wang2025mvp} addresses this problem to some extent by providing shortcuts from deeper layers to the concept vector.


Attention mechanism blocks provide an alternate strategy for producing intrinsically interpretable feature maps that somewhat resemble saliency maps \cite{linsley2018learning, xiao2015application, wang2017residual, woo2018cbam}.  These attention blocks are not to be confused with the self-attention that powers Transformers and Vision Transformers \cite{vaswani2017attention, dosovitskiy2020image}.  Notable attention blocks include GALA \cite{linsley2018learning} and CBAM \cite{woo2018cbam} which are simple easy-to-integrate attention modules, that can suppress irrelevant features and improve CNN performance.  Typically attention blocks include a spatial attention map and channelwise attention map.  Both of which are soft prediction masks which are internally used via scaled dot product to suppress irrelevant features and highlight relevant features.  The spatial attention mask shows a spatial region that is considered most important for the model prediction, but is often biased toward the predicted class.  

Additional recent efforts have produced neural network architectures with intrinsic class attribution.  CoDa-nets and their successor B-cos nets provide a mechanism to align feature weights with discriminating feature patterns \cite{bohle2021coda, bohle2024bcos, bohle2022bcos}. 
B-cos nets modify the linear operator $\textbf{\textit{w}}^T \textbf{\textit{x}}$ in order to suppress the response of features $\textbf{\textit{x}}$ that are unaligned with weights $\textbf{\textit{w}}$.  Given an input image, the B-cos network collapses the network architecture to an input-dependent linear transform, which can be used to produce class attribution maps.

An additional branch of \textit{Intrinsically Interpretable} architectures designed to learn class-specific features.  The Class-Specific Gate (CSG) \cite{liang2020training} is a notable technique that introduces a new loss function to encourage individual filters to learn class specific representations.  This gate monitors the activation late stage filters relative to class labels to encourage L1 sparsity.  However the dynamic assignment of classes to filters introduces additional non-convexity in the training landscape, making convergence difficult.  
ProtoPNets are another related method of learning class specific features intrinsically \cite{chen2019looks}.  This method learns a bank of class-dependent feature prototypes in the penultimate layer.  The penultimate features are then compared against these prototypes in a nearest-neighbor fashion. 

\subsection{Contributions}
\begin{itemize}
\item We show that end-to-end feature alignment improves intrinsic interpretability, enabling raw feature maps to exhibit class attribution and gradual morphing.  
\item We prove that FA-CNN penultimate feature maps are identical Grad-CAM saliency maps.  In particular, the 3D Global Average Pooling classifier head trivializes the calculation of gradients.
\item We prove theoretically that the dampened skip connection yields a property known as \textit{bounded increment through depth}.  This causes features to gradually morph layer by layer while maintaining alignment.
\item We compare the classification performance of FA-CNN versus ResNet-18 and ResNet-50, and demonstrate that FA-CNN is a performant classifier for the CIFAR-10 CIFAR-100 and ImageNet-100 subset datasets.
\item We analyze the quality of FA-CNN averaged penultimate feature maps versus permutation based post-hoc saliency methods.  We find that FA-CNN achieves good attribution performance but lower than permutation methods.

\end{itemize}

\section{Approach}
\label{approach}


Our approach is to design \textit{order-preserving} layers such as to prevent the network from unnecessarily shuffling neuron features, thereby discouraging the shuffling of semantic concepts in the feature space.  This \textit{order-preservation} aligns the semantic meaning of all of the features throughout the network, thereby improving intrinsic interpretability.  If \textit{order-preservation} is performed \textit{end-to-end} then raw feature maps will exhibit label attribution.  Our insight is that typical \textit{Linear} and \textit{Conv} layers do not preserve ordering.  This lack of ordering means that if a semantic concept maps to feature \textit{i} in layer \textit{L}, then it is highly unlikely that this concept will continue to map to feature \textit{i} in layer \textit{L+1}.  


\begin{figure}[b]
  \centering
\includegraphics[width=0.4\textwidth]{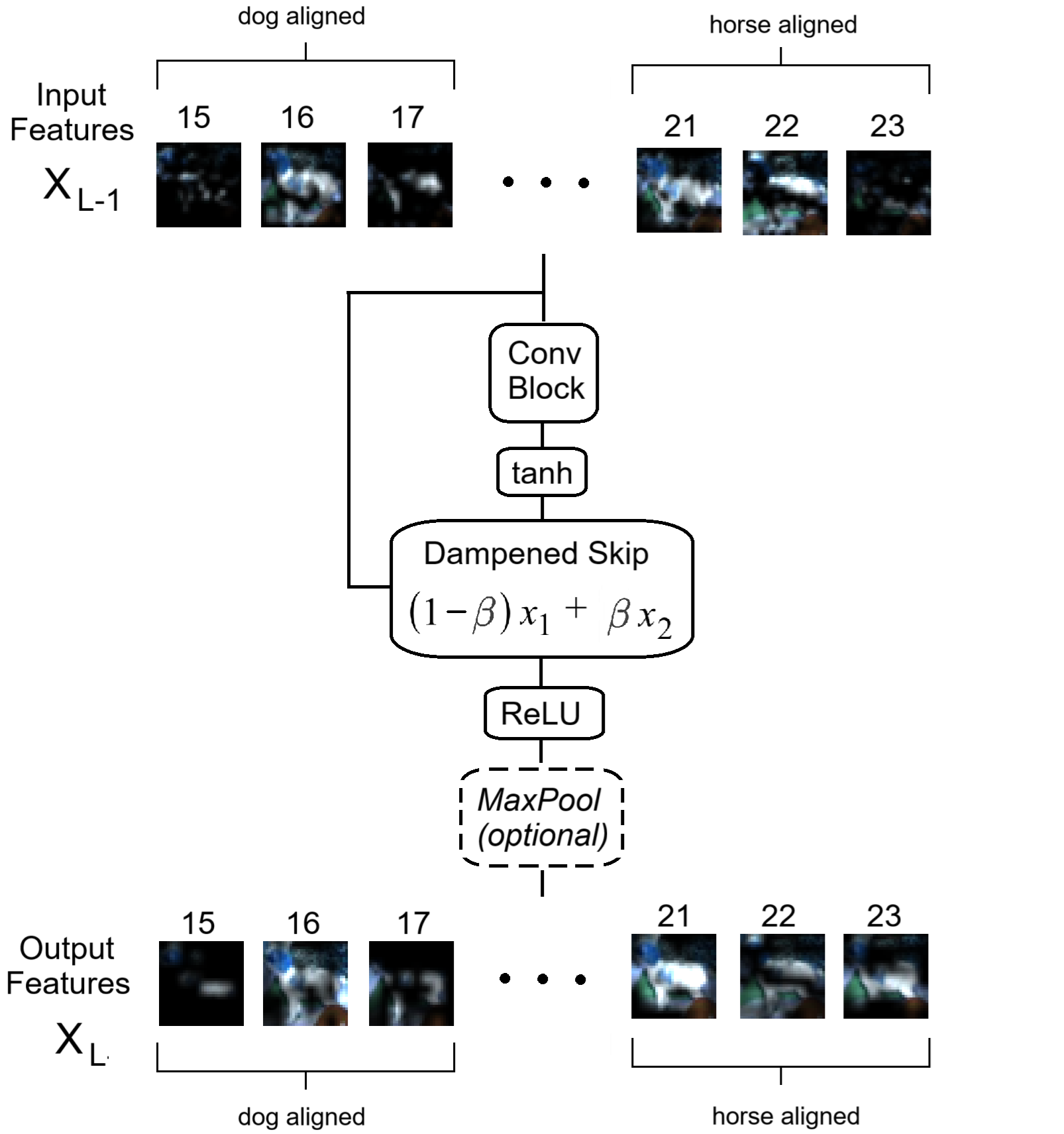}
    \caption{Diagram showing a single order-preserving FA-CNN Layer using the dampened skip connection.}
   \label{fig:dampened}
    \hfill
\end{figure}

Skip connections are a slight improvement in terms of order preservation, although inadequate to guarantee end-to-end alignment.  The traditional skip connection exhibits the following form:

\begin{equation}
X_{L+1} = ReLU(X_L + F(X_L))
\end{equation}

As can be seen, such a skip connection attempts to perform some amount of order-preservation although it is not sufficient to guarantee bounded increment through network depth.  The reason why skip connections 
do not guarantee bounded increment is because feature vector $X_L$ and output feature vector $X_{L+1}$ differ by at most the block outputs $F(X_L)$.  Note, absolute value is element wise as follows.

\begin{equation}
|X_{L+1} - X_{L}| \le |F(X_L)|
\end{equation}

Typically the block $F(X_L)$ is the result after Batch Normalization (BN).  Unfortunately, BN alone does not guarantee $F(X_L)$ is small or even normally distributed.  In practice the CNN may learn a highly skewed distribution for $F(X_L)$ in order to allow the block outputs to overpower the features $X_L$.  The ability of $F(X_L)$ to produce large magnitude values to overpower $X_L$ means that bounded increment is not guaranteed, and end-to-end feature alignment is not preserved.

\subsection{Dampened Skip Connection}

In order to 
guarantee bounded increment and thereby promote end-to-end feature alignment, we introduce a \textit{dampened skip connection}


\begin{equation}
\label{dampened}
\begin{aligned}
X_{L} = \textit{ReLU}\left( \left(1\!-\!\beta\right)X_{L-1} + \beta \; tanh\left(F\left(X_{L-1}\right)\!\mkern1mu\right)\!\mkern1mu\right) &\\
\textbf{where} \; \beta = \frac{1}{L} \quad \quad \quad \quad \quad \quad \quad \quad &
\end{aligned}
\end{equation}

The general idea is that $\beta \; tanh\left( F\left(X_L\right) \right)$ cannot exceed the interval of $[-\beta, \beta]$.  Using dampened skip connections, it is possible to construct a network that exhibits bounded increment though depth, thereby aligning features in consecutive layers.  In order to guarantee this property, however we must select the $\beta$ coefficient appropriately.  

In particular, we set $\beta$ as inversely proportional to the current layer depth $1 / L$ for layers $L=1 \dots N_L$.  Intuitively, this choice ensures that each successive layer has equal contribution to the overall feature magnitude.  Theoretically, this parameter choice is appropriate to guarantee the bounded increment property as proven in section \ref{sec:bounded}.  A diagram of the overall FA-CNN block including the dampened skip connection is shown in figure \ref{fig:dampened}.





\subsection{3D Global Average Pooling Classifier}

The next innovation of our network is to construct a classifier head to maintain feature alignment.  Typically the classifier head makes use of Linear layers that shuffle and mix together the features.  Our goal is to obtain an end-to-end alignment of the feature space from raw pixels all the way to label categories.  As such, it is desirable to have a strictly order-preserving classification head.

Our solution is to introduce a 3D Global Average Pooling head.  This head, not only averages the featuremaps spatially, but also averages consecutive filters together in order to predict the target categories.  As such, the pre-logit predictions are simply the 3D Global Average pooling of entire class-aligned penultimate feature maps as seen in figure \ref{fig:avgpool}.
This pooling ensures that feature maps $0 \dots R-1$ are predictive of label $0$, feature maps $R \dots 2R-1$ are predictive of label $1$ and so forth.  This classification head has a further justification because it is theoretically very similar to the underlying assumptions made by Grad-CAM, and this causes the FA-CNN penultimate feature maps to have a strong theoretical relation to Grad-CAM as we describe in section \ref{sec:gradcam}.

\begin{equation}
Y = \text{3DGlobalAvgPool}(X_{N_L})
\end{equation}

Once the pre-logits are calculated, we must further calculate the final logits, but as FA-CNN is a ReLU network, simply running softmax would not be appropriate, as the feature maps are non-negative.  Thus the pre-logits are also non-negative.  As such, we instead normalize the pre-logits based on the $L1$ norm in order to calculate final logits $\hat{Y}$

\begin{equation}
\label{prelogits}
\hat{Y} = \frac{Y}{||Y||_1}
\end{equation}



\begin{figure}[t]
\includegraphics[width=0.4\textwidth]{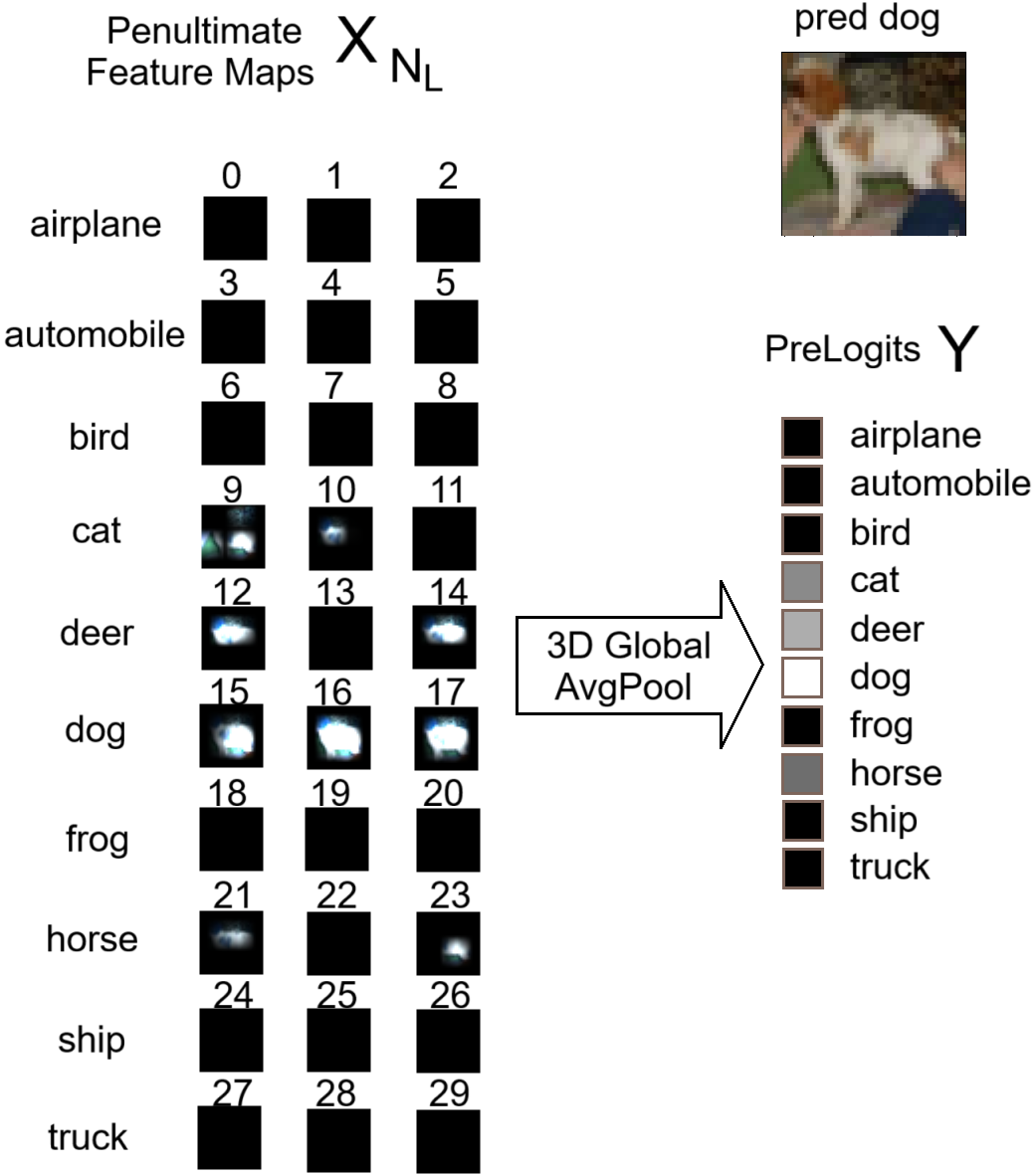}
    \caption{Diagram showing The 3D Global Average Pooling head to compute pre-logits from Penultimate Feature Maps.}
\label{fig:avgpool}
\end{figure}

\section{Proof of Grad-CAM equivalence}
\label{sec:gradcam}

We now prove that a trivial averaging of FA-CNN penultimate feature maps yields Grad-CAM saliency maps. 
As such, one does not need to run Grad-CAM separately for each of $C$ classes in post-hoc, because all $C$ Grad-CAM maps are already generated for free by the penultimate layer.  In order to facilitate legibility of the proof, we use the same notation as \cite{selvaraju2017grad}.  Notably, we have previously referred to the FA-CNN feature maps as $X_L$ for layers $1 \dots N_L$, thus the penultimate feature maps would be $X_{N_L}$.  For this proof, we define $A = X_{N_L}$ because this is standard notation for Grad-CAM.

\begin{thm} \label{thm:gradcam}
Grad-CAM saliency maps are obtained by simply averaging together the $R$ penultimate feature maps designated to class $c$ as follows.

\begin{equation}
L^c_{\text{Grad-CAM}}
\left(x, y\right) \; = \; \frac{1}{R} \sum_{\lfloor k  / R \rfloor = c} A_k \left(x, y\right)
\end{equation}
\end{thm}

\begin{proof}
Grad-CAM saliency maps are defined by the following equation \cite{selvaraju2017grad}.

\begin{equation}
L^{c}_{\text{Grad-CAM}}\left(x,y\right) = \text{ReLU}\left(\sum_k \alpha^c_k A_k\left( x,y \right) \right)
\label{gradcam}
\end{equation}

Where $A_k\left(x,y\right)$ is the raw penultimate feature map for filter $k$ at pixel location $(x,y)$. $\alpha^c_k$ is a class dependent weighting as determined by the global average contribution of the feature map to the pre-logit classifier score $Y^c$ as follows,

\begin{equation}
\alpha^c_k = {\underbrace{\frac{1}{Z}}_{\text{Global Avg-Pool}}} \sum_i \sum_j {\underbrace{\frac{\partial Y^c}{\partial A_k\left(i,j\right)}}_{\text{Gradient }Y^c_k\text{ w.r.t }A_k}}
\end{equation}

For a typical neural net, one must run run the model $C$ times with gradients, in order to calculate $\partial Y^c / \partial A_k$ for each output class $c=1 \dots C$.  These derivatives are used to generate the coefficients $\alpha^c_k$ needed to construct the Grad-CAM saliency maps.  

But for FA-CNN, the partial derivatives $\partial Y^c / \partial A_k$ are trivial, which further trivializes calculation of coefficients $\alpha^c_k$.  These partial derivatives are trivial, because FA-CNN uses 3D Global Average Pooling to calculate pre-logit score $Y^c$ from feature maps $A_k$.  Given $K=C R$ feature maps, where $R$ is the number of feature maps per output category, the partial derivatives are either zero (for feature maps irrelevant to the category) or positive constant (for feature maps relevant to output category) as follows,

\begin{equation}
\frac{\partial Y^c}{\partial A_k\left(i,j\right)} =
\begin{cases}
 1 / R & \text{if } \lfloor k  / R \rfloor = c  \\
  \; \; 0 & \; \; \ \text{otherwise}
\end{cases}
\end{equation}

Thus, $\alpha^c_k$ simplifies to $1 / R$ for each of the $R$ feature maps relevant to class $c$, and zero otherwise.  As $A_k$ is post-ReLU, both $A_k$ and $\alpha^c_k$ are non-negative.  Therefore the Grad-CAM like saliency maps in equation (\ref{gradcam}) are obtained by simply averaging together the $R$ penultimate raw feature maps relevant to class $c$, thereby arriving at the definition shown in Theorem \ref{thm:gradcam}.
\end{proof}

\section{Proof of Bounded Increment through Depth}
\label{sec:bounded}

We now prove our claim that \textit{``Feature maps smoothly morph through network depth from initial inputs all the way to the final Grad-CAM maps"}.  The practical benefit of this claim is that one can inspect the evolution of featuremaps through network depth toward the construction of GradCAM maps.  This aids in understanding which classes the model is considering at any given depth. 

In order to prove this claim, we must first define the claim in a bit more detail.  More specifically, we claim that a straight-forward resizing and rescaling of the raw FA-CNN featuremaps exhibits a property known as \textit{Bounded Increment through Depth}.  This means that the rescaled feature maps change at most a constant rate over depth.
We refer to the feature map at layer depth $L$ as  $X_L$ for $L=1 \dots N_L$.  We define our \textit{rescaled} feature maps as $S_L^{(T)}$ as follows,

\begin{equation}
S_L^{(T)} = \begin{cases}
\quad \quad \quad 0  & \quad \text{ if } L = 0\\
\; \frac{L}{N_L} \; \text{Pool} \big( X_L, T\big) & \quad otherwise
\end{cases}
\label{rescale}
\end{equation}

Where $T$ is a target layer for downsampling all features to the same resolution.  As FA-CNN internal layers use MaxPooling, this downsampling must also use MaxPooling.  This proof can also be adapted to use AvgPooling, given a variant of FA-CNN that uses AvgPooling internally.  For FA-CNN,
$\text{Pool}\left(X_L, T\right)$ is defined as follows.


\begin{equation}
\begin{aligned}
\text{Pool}\left( X_L, T \right) = \quad & \\ \text{MaxPool} & \left(X_L,   \; \text{ker}=\left(\frac{\text{rows}_L}{\text{rows}_{T}},  \frac{\text{cols}_L}{\text{cols}_{T}}\right) \right)
\end{aligned}
\end{equation}

\begin{thm}
\label{thm:bounded}
Rescaled feature maps exhibit bounded increment through depth.  As such rescaled feature values of consecutive layers $L$ and $L+1$ change by at most a constant factor $\delta = 1 / {N_L}$ as follows,.

\begin{equation}
\begin{aligned}
\forall \; 1 \le L \le T \quad \quad \quad \quad \quad \quad \quad \quad & \\
\left| S_{L,k}^{(T)}(x,y)  - S_{L-1,k}^{(T)}(x,y) \right| & \le \delta \quad \\ 
\textbf{where } \delta = \frac{1}{N_L} \quad \quad \quad &\; 
\end{aligned}
\end{equation}
\end{thm}



\begin{lem}
\label{lem:bounded}
Rescaled features at layer $L$ are bound to a non-negative range $[0, \; \delta L]$
\begin{equation}
\forall \; 0 \le L \le N_L \quad \quad \quad S_{L,k}^{(N_L)}(x,y) \; \in \;  \left[ 0, \; \delta L \right]
\end{equation}
\end{lem}
\begin{proof} 
We will prove Theorem \ref{thm:bounded} and Lemma \ref{lem:bounded} simultaneously by induction over variable $T$. For brevity we omit indices $k, x, y$ but refer these implicitly. \vspace{3pt}

\noindent\underline{Base Case $\; T=0$}:  Clearly if $T=0$ then $L=0$, so trivially $S_0^{(0)}=0$ by equation \ref{rescale}.  So Lemma \ref{lem:bounded} is proven for $T=0$.\vspace{4pt}

\noindent\underline{Base Case $\; T=1$}:  Clearly $S_0^{(T)}=0$ for all $T$ according to equation \ref{rescale}.  Featuremap $X_1$ is post \textit{ReLU} and post \textit{tanh}, therefore $X_1 \in \left[0, 1\right]$.  For $T=L$ no pooling is performed.  Thus, $S_1^{(1)} = \delta X_1$ by equation \ref{rescale}.  Thus, $X_1 \in \left[0, \delta \right]$, so Lemma \ref{lem:bounded} is proven for $T=1$. \vspace{4pt}

As $S_1^{(1)} \in \left[0, \delta\right]$ by Lemma \ref{lem:bounded} and $S_0^{(1)}=0$ by equation \ref{rescale} it must be $\left|S_1^{(1)} - S_0^{(1)} \right| \le \delta$, so Theorem \ref{thm:bounded} is proven for $T=1$. \vspace{4pt}

\noindent\underline{Inductive Case.}  \textit{Given Theorem \ref{thm:bounded} and Lemma \ref{lem:bounded} are true for $T-1$, prove Theroem \ref{thm:bounded} and Lemma \ref{lem:bounded} are true for $T$}. \vspace{3pt}
Given: 
\begin{equation}
\begin{aligned}
\forall \; 1 \le L \le T-1 \quad \quad 
\left| S_{L}^{(T-1)}  - S_{L-1}^{(T-1)} \right| & \le \delta \quad 
\end{aligned}
\end{equation}

By property 1 (supplemental material) we must have,
\begin{equation}
\begin{aligned}
\forall \; 1 \le L \le T-1 \quad \quad 
\left| S_{L}^{(T)}  - S_{L-1}^{(T)} \right| & \le \delta \quad 
\end{aligned}
\end{equation}

So Theorem \ref{thm:bounded} only remains to be proven for the target layer $L=T$.  Thus, we assume $L=T$ for the remainder of the proof.  Note that layer $L-1=T-1$ so equation \ref{rescale} does not perform pooling in this case.  Thus we can rearrange equation \ref{rescale} as follows,


\begin{equation}
\label{rearrange}
X_{L-1} = \frac{N_L}{L-1} S_{L-1}^{(T-1)}
\end{equation}

And by equation \ref{dampened}, before max pooling,
\begin{equation}
\begin{aligned}
X_{L} = \textit{ReLU}\left( \left(1\!-\!\beta\right)X_{L-1} + \beta \; tanh\left(F\left(X_{L-1}\right)\!\mkern1mu\right)\!\mkern1mu\right) &\\
\textbf{where} \; \beta = \frac{1}{L} \quad \quad \quad \quad \quad \quad \quad \quad &
\end{aligned}
\end{equation}

By substituting equation \ref{rearrange},
\begin{equation}
\begin{aligned}
\frac{N_L}{L} S_L^{(T-1)} = 
\textit{ReLU} \left( \frac{L-1}{L}\frac{N_L}{L-1} S_{L-1}^{(T-1)} + J \right) & \\
\textbf{where} \; J \in \left[ -\beta, \beta \right] \quad \quad \quad \quad \quad \quad &
\end{aligned}
\end{equation}

After simplification,
Given $S_{L-1}^{(T)} \in $
\begin{equation}
\label{simplification}
S_L^{(T-1)} = \textit{ReLU}\left( S_{L-1}^{(T-1)} + H \right) 
\textbf{where} \; H \in \left[ -\delta, \delta \right]
\end{equation}

Thus, $\left| S_L^{(T-1)} - S_{L-1}^{(T-1)}\right| \le \delta$
and given property 1 (supplemental material), 

\begin{equation}
\label{thmproof}
\left| S_L^{(T)} - S_{L-1}^{(T)}\right| \le \delta
\end{equation}

and Theorem \ref{thm:bounded} is proven.

From here, we easily prove Lemma \ref{lem:bounded}, because $S_{L-1}^{(T)} \in [0, \delta(L-1)]$ by induction, and given equation \ref{thmproof}, and given that $S_L^{(T)}$ is post ReLU by equation \ref{simplification}.  Therefore,
\begin{equation}
S_{L-1}^{(T)} \in \left[ 0, \delta L \right]
\end{equation}

And Lemma \ref{lem:bounded} is also proven.
\end{proof}

\section{Results}
\label{sec:results}

We evaluate the proposed FA-CNN on three image classification benchmarks of increasing complexity: CIFAR-10 (10 classes, $32\times32$), CIFAR-100 (100 classes, $32\times32$), and a randomly selected ImageNet-100 subset (100 classes, native ImageNet resolution). All experiments report top-1 validation accuracy.

All models are trained from scratch without any pretraining with a batch size of 100 training samples. 
For data augmentation, the CIFAR datasets employ standard random cropping and horizontal flipping. 
FA-CNN is configured with 24 layers with MaxPooling after every 8 layers.  FA-CNN uses 200 channels for CIFAR-10, but
the number of channels is increased to 600 for CIFAR-100 and ImageNet-100, maintaining approximately six feature maps per class for the Pooling head. 
For CIFAR datsets the models are trained for 300 epochs, whereas for ImageNet-100, the models were stopped at 100 epochs as the training and validation curves indicate convergence.

For this comparison, we adopt a minimal optimization setting using stochastic gradient descent (SGD). 
We deliberately disable momentum and weight decay for FA-CNN in our main experiments to provide a controlled comparison that highlights the architectural properties of the model. This setup ensures that the reported gains primarily reflect structural design rather than optimization-specific advantages.  
We separately include results for the ResNet baselines following standard training protocols found in literature to ensure competitive performance. 
For CIFAR-10 and CIFAR-100, we use SGD with momentum $0.9$ and weight decay $5\times10^{-4}$, together with a MultiStep learning rate schedule (milestones at epochs 100 and 150, decay factor $0.1$). For ImageNet-100, we use SGD with momentum $0.9$, weight decay $1\times10^{-4}$, and a Step learning rate schedule (step size 30, decay factor $0.1$).  

We compare three variants of FA-CNN with standard ResNet baselines.  FA-CNN (Pooling head) is the full FA-CNN architecture with \textit{dampened skip connections} and \textit{the global average pooling head}.  FA-CNN (Linear head) still has the \textit{dampened skip connections} but a standard linear classification head.  Vanilla-CNN is a disabled version of FA-CNN with just regular skip connections and the standard linear head.  All three variants have the exact same backbone configuration and layers.



\subsection{Classification Accuracy}


\begin{table}[b]
\centering
\renewcommand{\arraystretch}{1.25}  
\small                              
\resizebox{\columnwidth}{!}{
\begin{tabular}{lccc}
\toprule
\textbf{Method} & \textbf{CIFAR-10} & \textbf{CIFAR-100} & \textbf{ImageNet-100} \\
\midrule
FA-CNN (\textbf{Pooling} head) & 93.13 & \textbf{77.03} & \textbf{81.78} \\
FA-CNN (Linear head)           & \textbf{93.80} & 75.32 & 77.50 \\
Vanilla-CNN           & {92.97} & 75.07 & 76.98 \\

ResNet-18                     & 93.57 & 73.96 & 77.74 \\
ResNet-50                      & 91.50 & 72.45 & 76.76 \\

\midrule
ResNet-18 (standard recipe)                  & \textbf{94.89} & 77.31 & \textbf{83.44} \\
ResNet-50 (standard recipe)     & 94.49 & \textbf{78.10} &  82.96\\
\bottomrule
 \end{tabular}}
\caption{\textbf{Top-1 validation accuracy (\%)} across datasets for all models.  Models with minimal training recipe (top) reference models with standard recipe (bottom). 
}

\label{tab:main_results}
\end{table}

Table~\ref{tab:main_results} summarizes the top-1 validation accuracy across all datasets.  Table \ref{tab:main_results} (top) shows the comparison of all models with the minimal training recipe.  FA-CNN achieves accuracy comparable to standard ResNets despite incorporating architectural constraints aimed at improving interpretability. 
On \text{CIFAR-10}, FA-CNN with the Linear head attains \(93.8\%\) accuracy, matching ResNet-18 (\(93.57\%\)) and exceeding ResNet-50 (\(91.5\%\)). 
On \text{CIFAR-100}, FA-CNN with the Pooling head achieves \(77.03\%\), exceeding ResNet-18 (\(73.96\%\)) while providing intrinsic class attribution.
On the \text{ImageNet-100} subset, FA-CNN reaches \(81.78\%\), comparable to the minimal-trained ResNet-18 baseline (\(77.74\%\)). 
In comparison to ResNet with the standard recipe, we see that the FA-CNN (minimal) results are between $1\%$ to $2\%$ lower than the best ResNet model.

These results confirm that both the dampened skip connection and interpretable pooling actually improve performance relative to an identical Vanilla-CNN architecture, and with minimal training can even come close to ResNet models with the standard recipe.

\subsection{Pixel-Removal Saliency Evaluation}


We quantitatively assess the faithfulness of FA-CNN’s intrinsic saliency maps using the standard \emph{pixel-removal} protocol widely adopted in interpretability benchmarks. For each input image, a given percentage of the most salient pixels is progressively removed, and the resulting drop in top-1 accuracy is measured. Varying the removal fraction yields an \emph{accuracy–fraction removed} curve that reflects the stability of the attribution. We compare FA-CNN against Grad-CAM \cite{selvaraju2017grad}, \text{RISE}~\cite{petsiuk2018rise} and \text{K-Occlusion} (K-OCC)~\cite{simonyan2013deep}.  We also include random masking baseline (RAND).

Figure~\ref{fig:facnn_block} shows the pixel-removal robustness curves on ImageNet-100 for intermediate Layer 18 and the penultimate layer 24. It is clear that for both FA-CNN and Grad-CAM the best results are obtained using the penultimate layer (Layer 24), and the saliency maps for the intermediate layer (Layer 18) are less robust for this task.  Perturbation methods RISE and K-OCC have no layer dependence because these exhaustive black-box methods are unaware of the architecture.
All methods are skillful relative to random removal.  In the low-removal range (0--20\%), perturbation-based methods, particularly RISE and K-OCC, typically exhibit the strongest robustness. This robustness is due to the permutation methods' ability to increase classification accuracy when a small percentage of irrelevant pixels are removed.  FA-CNN and Grad-CAM are unable to increase classification accuracy in the low-removal range, but exhibit similar robustness to RISE and K-OCC in the moderate removal range (40--60\%) and exceeds K-OCC in the high removal range (60\%+) showing a strong ability to prevent relevant pixels from being removed.


\begin{figure}[tbp]
  \centering
\includegraphics[width=0.45\textwidth]{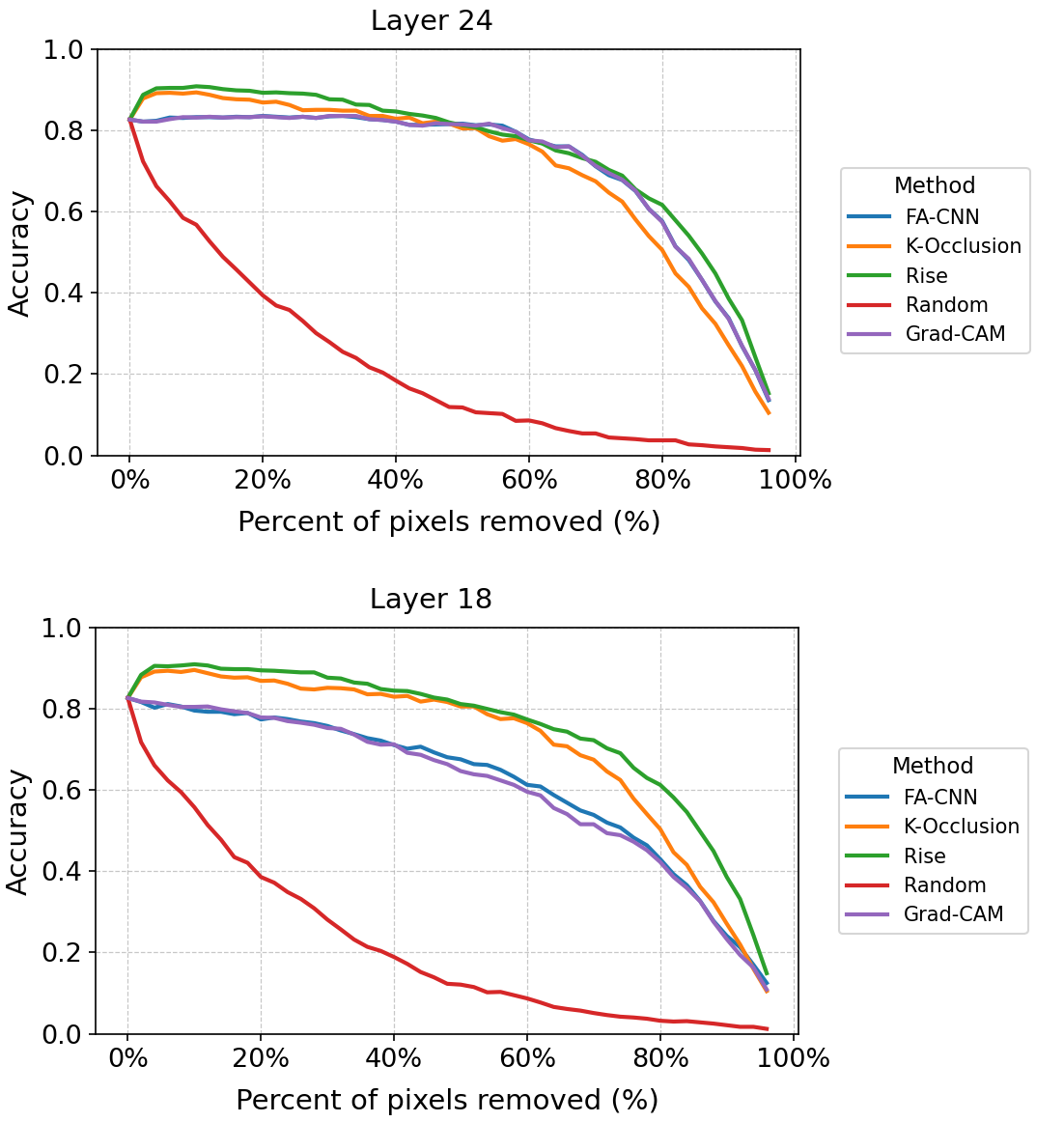}
    \caption{Classification Accuracy with Percent Pixels Removed on Layer 24 (top) and Layer 18 (bottom) on ImangeNet-100 Dataset.}
    \label{fig:facnn_block}
\end{figure}

\begin{figure*}[htbp]
  \centering


\includegraphics[width=1\textwidth]{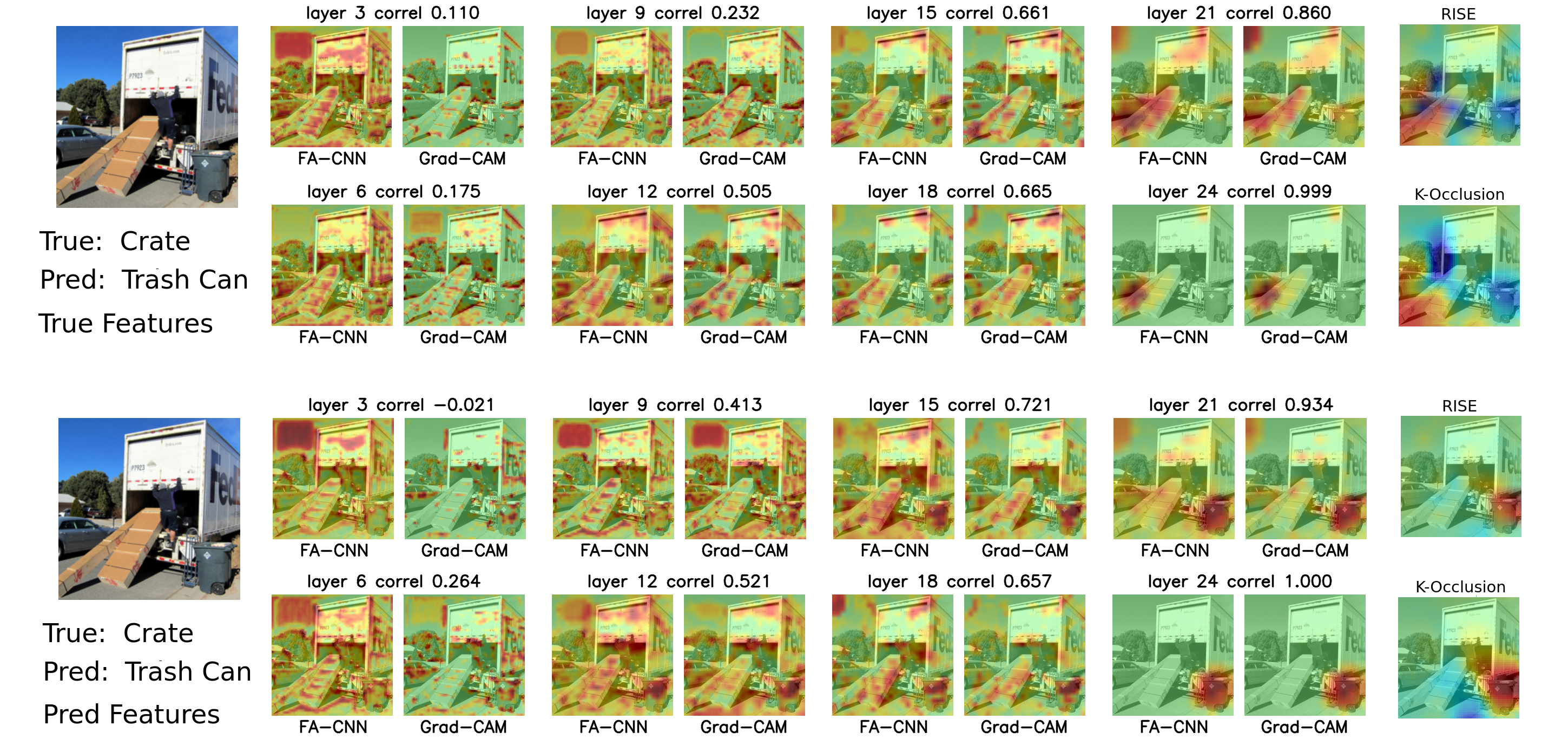}
    \caption{Comarison of averaged FA-CNN raw featuremaps (Left) versus Grad-CAM maps (Right) and Perturbation methods (Far right).}
    \label{fig:qualitative_wrong}
    \hfill
\end{figure*}



\subsection{Qualitative Evaluation}

Figures \ref{fig:billboard} and \ref{fig:qualitative_wrong} show examples of averaged FA-CNN feature maps versus Grad-CAM, RISE, and K-Occlusion.  Penultimate FA-CNN feature maps show perfect correlation with Grad-CAM.  Intermediate feature maps show gradual morphing and provide insight as to which classes the model is considering at various depths.  In Figure \ref{fig:billboard}, the FA-CNN feature maps for the \textit{Stork} class show very good agreement with Grad-CAM in layers 12 and onward.  We also see that the FA-CNN featuremaps for the Stork class are most active over the head and upper leg area. This result qualitatively agrees with RISE, although K-Occlusion appears to only identify the head region. Moreover we see that the result from FA-CNN Layer 18 also very closely agrees with RISE over specific areas of the head and upper leg.

Figure \ref{fig:qualitative_wrong} shows an interesting result in that the image is of a shipping crate, but there is also a trash can in the scene.  Additional examples are in Supplemental Materials.  All methods correctly show that the shipping crate region is important for correctly classifying the image as a crate, and that the trash can influential in steering the network toward  misclassification.
We see that the earlier layers are quite noisy in distinguishing these classes showing that crates and trash cans are relatively high-level semantic concepts that require great model depth.  The FA-CNN and Grad-CAM results show strong correlation starting layer 21, although visual similarity can be seen in layer 9 onward.  We also observe that perturbation methods also identify negative attributions \textit{blue}, whereas FA-CNN attribution maps only show positive attributions \textit{red}.  These negative attributions may be why the perturbation methods are able to increase accuracy on the percent pixels removed task, whereas Grad-CAM and FA-CNN raw features are unable to do so.


\section{Discussion and Future Work}
\label{sec:discussion}

We have shown that end-to-end feature alignment is a simple and viable strategy to improve the interpretability of raw CNN feature maps.  We prove two theoretical results.  1. Class aligned FA-CNN penultimate feature maps can be simply averaged together to produce Grad-CAM saliency maps.  2. Rescaled FA-CNN features exhibit a property known as \textit{bounded increment through depth} which is useful for model refinement and visualization.  We find that qualitatively, gradual morphing is very useful in determining the \textit{required network depth} for the model to identify a labeled class.  We find that the dampened skip connections and global average pooling actually slightly improve classification performance relative to an identical Vanilla-CNN design.  Moreover performance is slightly higher than ResNet when trained using an identical training recipe, and slightly lower as compared to a stronger standard recipe.
Furthermore we show that FA-CNN feature maps achieve good qualitative performance on the explainability task of zeroing irrelevant pixels, although this performance is somewhat lower than exhaustive perturbation methods.

A limitation of our method is the inability to identify negative class attribution.  There are several possible approaches that might be applicable to this problem including the use of positive attribution in other classes as evidence of negative attribution to a target class.  Alternatively, it may be possible to explore the use of alternative activation functions other than ReLU for the feature representation.  Typical activation functions including ReLU, Leaky-ReLU and others all have the side effect of either restricting the domain to the non-negative features, and/or causing the distribution of features to become highly non-normal \cite{chapman2024interpretable}.  It is possible that alternative activation functions could allow the deep features to intrinsically learn negative attributions in addition to positive attributions within the raw feature space. 

Toward future work, an important question is whether the FA-CNN approach would be equally applicable to Vision Transformer architectures.  We do not see any fundamental reason why not, because much like CNN architectures, Vision Transformer architectures also intrinsically make use of skipped connections before and after MSA and MLP blocks that could be replaced with the proposed dampened skip connections.  Moreover, many vision transformer classifiers make use of Global Average Pooling of the tokens immediately prior to a classification head, which is indeed in the spirit of our design.  If the final classifier were replaced by pooling we believe this would push the MLP blocks to learn how to separate the classes directly.  It would also be interesting to determine if such a hybrid CNN/ViT design could achieve class-specific attention rollout, which is presently a challenging problem with ViT class attribution.  One of the problems with attention rollout for ViT architectures, is that the attention heads are not class specific.  But if the token features were intrinsically class-aligned as similar to FA-CNN, one could in theory design class-specific attention heads, and then trivially perform attention rollout separately for each class.

\subsection*{Acknowledgment}
We would like to thank the Frost Institute for Data Science and Computing for their support.

{
    \small
    \bibliographystyle{ieeenat_fullname}
    \bibliography{main}
}


\end{document}